\documentclass[journal]{IEEEtran}

\usepackage{xcolor,soul,framed} 

\colorlet{shadecolor}{yellow}
\usepackage[pdftex]{graphicx}
\graphicspath{{../pdf/}{../jpeg/}}
\DeclareGraphicsExtensions{.pdf,.jpeg,.png}
\usepackage{cite}
\usepackage[cmex10]{amsmath}
\usepackage{array}
\usepackage{mdwmath}
\usepackage{mdwtab}
\usepackage{eqparbox}
\usepackage{url}
\usepackage{float}
\usepackage{stfloats}
\usepackage[export]{adjustbox}
\usepackage{wrapfig}
\usepackage{amsfonts}
\usepackage{amsmath}
\usepackage{changes}
\definechangesauthor[name={Sami}, color=red]{sk}
\usepackage{braket}
\usepackage{bm} 
\newcommand{\vect}[1]{\boldsymbol{\mathbf{#1}}}

\usepackage{multicol}
\usepackage{comment}
\usepackage{subcaption}
\usepackage[utf8]{inputenc} 
\usepackage[T1]{fontenc}    
\usepackage{hyperref}       
\usepackage{url}            
\usepackage{booktabs}       
\usepackage{amsfonts,amsmath}
\usepackage{nicefrac}  
\usepackage{microtype}      
\usepackage{graphicx}
\usepackage{color}
\usepackage[linesnumbered,ruled]{algorithm2e}

\newtheorem{theorem}{Theorem}
\newtheorem{definition}{Definition}

\begin{document}
\bstctlcite{IEEEexample:BSTcontrol}
    \title{A Gradient-Aware Search Algorithm for Constrained Markov Decision Processes}
  \author{Sami~Khairy,
      Prasanna~Balaprakash,
     Lin X. Cai
  \thanks{Sami Khairy and Lin X. Cai are with the Department of Electrical and Computer Engineering, Illinois Institute of Technology, Chicago, IL 60616, USA.
E-mail: skhairy@hawk.iit.edu, lincai@iit.edu.}
  \thanks{Prasanna Balaprakash is affiliated with the Mathematics and Computer Science Division and Leadership Computing Facility, Argonne National Laboratory,  Lemont, IL 60439, USA. E-mail: pbalapra@anl.gov.}%
 }  

\maketitle

\begin{abstract}
The canonical solution methodology for finite constrained Markov decision processes (CMDPs), where the objective is to maximize the expected infinite-horizon discounted rewards subject to the expected infinite-horizon discounted costs constraints, is based on convex linear programming. In this brief, we first prove that the optimization objective in the dual linear program of a finite CMDP is a piece-wise linear convex function (PWLC) with respect to the Lagrange penalty multipliers. Next, we propose a novel two-level Gradient-Aware Search (GAS) algorithm which exploits the PWLC structure to find the optimal state-value function and Lagrange penalty multipliers of a finite CMDP. The proposed algorithm is applied in two stochastic control problems with constraints: robot navigation in a grid world and solar-powered unmanned aerial vehicle (UAV)-based wireless network management. We empirically compare the convergence performance of the proposed GAS algorithm with binary search (BS), Lagrangian primal-dual optimization (PDO), and Linear Programming (LP). Compared with benchmark algorithms, it is shown that the proposed GAS algorithm converges to the optimal solution faster, does not require hyper-parameter tuning, and is not sensitive to initialization of the Lagrange penalty multiplier. 
\end{abstract}


\IEEEpeerreviewmaketitle

\section{Introduction}

\IEEEPARstart{M}{arkov} decision processes (MDPs) are classical formalization of sequential decision making in discrete-time stochastic control processes \cite{sutton2018reinforcement}. In MDPs, the outcomes of actions are uncertain, and influence not only immediate rewards, but also future rewards through next states. Policies, which are strategies for action selection, should therefore strike a trade-off between immediate rewards and delayed rewards, and be optimal in some sense. MDPs have gained recognition in diverse fields such as operations research, economics, engineering, wireless networks, artificial intelligence, and learning systems \cite{kallenberg2011markov}. Moreover, an MDP is a mathematically idealized form of the reinforcement learning problem, which is an active research field within the machine learning community. In many situations however, finding the optimal policies with respect to a single reward function 
does not suffice to fully describe sequential decision making in problems with multiple conflicting objectives \cite{mannor2004geometric,van2014multi}.  
The framework of constrained MDPs (CMDPs) is the natural approach for handling multi-objective decision making under uncertainty \cite{altman1999constrained}.

Algorithmic methods for solving 
CMDPs have been extensively studied when the underlying transition probability function is known \cite{altman1999constrained,dolgov2005stationary,zadorojniy2006robustness,piunovskiy2000constrained,chen2004dynamic,piunovskiy2006dynamic,chen2007non}, and unknown \cite{geibel2005risk,geibel2012learning,borkar2005actor,abad2002self,bhatnagar2010actor,bhatnagar2012online,tessler2018reward,liang2018accelerated,fu2018risk,achiam2017constrained,prashanth2014actor,paternain2019safe,ding2020provably,raybenchmarking,khairy2020constrained}. In the case of finite state-action spaces with known transition probability function, the solution for a CMDP can be obtained by solving finite linear programs \cite{altman1999constrained,zadorojniy2006robustness,dolgov2005stationary}, or by deriving a Bellman optimality equation with respect to an augmented MDP whose state vector consists of two parts: the first part is the state of the original MDP, while the second part keeps track of the cumulative constraints cost \cite{piunovskiy2000constrained, piunovskiy2006dynamic,chen2004dynamic,chen2007non}. Linear programming methods become computationally impractical at a much smaller number of states than dynamic programming, by an estimated factor of about 100 \cite{sutton2018reinforcement}. In practice, MDP-specific algorithms, which are dynamic programming based methods, hold more promise for efficient solution \cite{littman2013complexity}. While MDP augmentation based methods provide a theoretical framework to apply dynamic programming to constrained stochastic control, they introduce continuous variables to the state space which rules out practical tabular methods and make the design of a solution algorithm challenging \cite{chow2015risk}.

On the other hand, solution methods for the constrained reinforcement learning problem, i.e., CMDPs with unknown transition probability, are generally based on Lagrangian primal-dual type optimization. In these methods, gradient-ascent is performed on state-values at a fast time scale to find the optimal value function for a given set of Lagrangian multipliers, while gradient-descent is performed on the Lagrangian multipliers at a slower time scale. This process is repeated until convergence to a saddle point. Existing works have explored the primal-dual optimization approach in the tabular setting, i.e., without function approximation
\cite{geibel2005risk,geibel2012learning,borkar2005actor,abad2002self}, and with function approximators such as deep neural networks \cite{bhatnagar2010actor,bhatnagar2012online,prashanth2014actor,tessler2018reward,liang2018accelerated,fu2018risk,achiam2017constrained,raybenchmarking,ding2020provably,paternain2019safe,khairy2020constrained}. While this approach is appealing in its simplicity, it suffers from the typical problems of gradient optimization on non-smooth objectives, sensitivity to the initialization of the Lagrange multipliers, and convergence dependency on the learning rate sequence \cite{chow2017risk, achiam2017constrained,liu2019ipo, liang2018accelerated}. If the learning rate is too small, the Lagrange multipliers will not update quickly to enforce the constraint; and if it is too high, the algorithm may oscillate around the optimal solution. In practice, a sequence of decreasing learning rates should be adopted to guarantee convergence \cite{borkar2005actor}, yet we do not have an obvious method to determine the optimal sequence, nor we can assess the quality of a solution in cases where the objective is not differentiable at the optimal solution.

In this brief, we develop a new approach to solve finite CMDPs with discrete state-action space and known probability transition function. We first prove that the optimization objective in the dual linear program of a finite CMDP is a piece-wise linear convex function (PWLC) with respect to the Lagrange penalty multipliers. Next, we treat the dual linear program of a finite CMDP as a search problem over the Lagrange penalty multipliers, and propose a novel two-level Gradient-Aware Search (GAS) algorithm, which exploits the PWLC structure to find the optimal state-value function and Lagrange penalty multipliers of a CMDP. We empirically compare the convergence performance of the proposed GAS algorithm with binary search (BS), Lagrangian primal-dual optimization (PDO), and Linear Programming (LP), in two application domains, robot navigation in grid world and wireless network management. Compared with benchmark algorithms, we show that the proposed GAS algorithm converges to the optimal solution faster, does not require hyper-parameter tuning, and is not sensitive to initialization of the Lagrange penalty multiplier. 

The  remainder  of  this  paper  is  organized  as follows.  A  background of unconstrained  and constrained  MDPs  is  given  in Section II. Our proposed Gradient-Aware Search (GAS) algorithm is proposed in Section III. Performance evaluation of GAS in two application domains is presented in Section IV, followed by our concluding remarks and future work in Section V. 

\section{Background and Related Works}
\subsection{Unconstrained Markov Decision Processes}
An infinite horizon Markov Decision Process (MDP) with discounted-returns is defined as a tuple $(\mathcal{S},\mathcal{A},\mathcal{P}, \vect{\beta},\mathcal{R},\gamma)$,  where $\mathcal{S}$ and $\mathcal{A}$ are finite sets of states and actions, respectively, $\mathcal{P}:\mathcal{S}\times \mathcal{A} \times \mathcal{S} \rightarrow [0,1]$ is  the model's state-action-state transition probabilities, and $\vect{\beta}:\mathcal{S} \rightarrow [0,1]$ is the initial distribution over the states, $\mathcal{R}: \mathcal{S} \times \mathcal{A} \rightarrow \mathbb{R}$, is the reward function which maps every state-action pair to the set of real numbers $\mathbb{R}$, and $\gamma$ is the discount factor. Denote the transition probability from state $s_t=i$ to state $s_{t+1}=j$ if action $a_t=a$ is chosen by $P_{ij}(a) := P(s_{t+1} = j | s_t = i, a_t = a)$. The transition probability from state $i$ to state $j$ is therefore, $p_{ij} = P(s_{t+1} = j | s_t = i) = \sum_a P_{ij}(a) \pi(a|i)$, where $\pi(a|i)$ is the adopted policy. The state-value function of state $i$ under policy $\pi$, $V_\pi(i)$, is the expected discounted return starting in state $i$ and following $\pi$ thereafter,
\begin{equation} \label{stateValue}
\resizebox{1\hsize}{!}{$
V_\pi(i) = \sum_{t=1}^\infty  \sum_{j, a} \gamma^{t-1} P^\pi(s_t = j, a_t = a | s_0 = i) \mathcal{R}(j,a), \forall i \in \mathcal{S}$.}
\end{equation}

Let $\vect{V}_\pi$ be the vector of state values, $V_\pi(i), \forall i$. The solution of an MDP is a Markov stationary policy $\pi^*$ which maximizes the inner product $\langle  \vect{V}_\pi, \vect{\beta} \rangle = \sum_i V_\pi(i) \beta(i)$, i.e., 
\begin{equation}\label{eq:sys1}
\underset{\pi}{\text{max}}~~ \sum_{t=1}^\infty \sum_{j, a} \gamma^{t-1} P^\pi(s_t = j, a_t = a) \mathcal{R}(j, a).
\end{equation}
There exist several methods to solve \eqref{eq:sys1}, including linear programming \cite{kallenberg2011markov} and dynamic programming methods such as value iteration and policy iteration \cite{sutton2018reinforcement}. Based on the linear programming formulation, the optimal $\vect{V}_{\pi^*}$ 
can be obtained by solving the following primal linear program \cite{kallenberg2011markov},
\begin{equation}\label{eq:primal1}
\begin{aligned}
\underset{V(i)}{\text{min}} ~~~~~~ &\sum_{\forall i} \beta(i)V(i), \\
&V(i) \geq \mathcal{R}(i,a) + \gamma \sum_j P_{ij}(a)V(j), \forall (i ,a) \in \mathcal{S} \times \mathcal{A}.
\end{aligned}
\end{equation}
Linear program \eqref{eq:primal1} has $|\mathcal{S}|$ variables and $|\mathcal{S}| \times |\mathcal{A}|$ constraints, which becomes computationally impractical to solve for MDPs with large state-action spaces. On the contrary, dynamic programming is comparatively better suited to handling large state spaces than linear programming, and is widely considered the only feasible way of solving \eqref{eq:sys1} for large MDPs \cite{sutton2018reinforcement}. Of particular interest is the value iteration algorithm which generates an optimal policy in polynomial time for a fixed $\gamma$ \cite{littman2013complexity}. In the value iteration algorithm, the state-value function is iteratively updated based on Bellman's principle of optimality,
\begin{equation} \label{valIter}
V_\pi^{k+1} (i) =\underset{a \in \mathcal{A}(i)}{\text{max}} [ \mathcal{R} (i, a) + \gamma \sum_{j \in \mathcal{S}} P_{ij}(a) V_\pi^k (j) ] , \forall i \in \mathcal{S}.
\end{equation}
It has been shown that the non-linear map $\vect{V}_\pi \rightarrow \vect{V}_\pi$ in \eqref{valIter} is a  monotone contraction mapping that admits a unique fixed point, which is the optimal value function $\vect{V}_{\pi^*}$ \cite{kallenberg2011markov}. By obtaining the optimal value function, the optimal policy can be derived using one-step look-ahead: the optimal action at state $i$ is the one that  attains the equality in \eqref{valIter}, with ties broken arbitrarily.

\subsection{Constrained Markov Decision Processes}
In constrained MDPs (CMDPs), an additional immediate cost function $\mathcal{C}: \mathcal{S} \times  \mathcal{A} \rightarrow R$ is augmented, such that a CMDP is defined by the tuple $(\mathcal{S},\mathcal{A},\mathcal{P}, \vect{\beta},\mathcal{R},\mathcal{C}, \gamma)$ \cite{altman1999constrained}.  The state-value function is defined as in \eqref{stateValue}. In addition, the infinite-horizon discounted-cost of a state $i$ under policy $\pi$ is defined as,
\begin{equation}
\resizebox{1\hsize}{!}{$
C_\pi(i) = \sum_{t=1}^\infty \sum_{j, a} \gamma^{t-1}  P^\pi(S_t = j, A_t = a | S_o = i) \mathcal{C}(j, a), \forall i \in \mathcal{S} $.}
\end{equation}
Let $\vect{C}_\pi$ be the vector of state costs, $C_\pi(i), \forall i$, and $E \in \mathbb{R}$ a given constant which represents the constraint upper-bound. The solution of a CMDP is a Markov stationary policy $\pi^*$ which maximizes $\langle \vect{V}_\pi, \vect{\beta} \rangle$  subject to a constraint $\langle \vect{C}_\pi, \vect{\beta} \rangle \leq E$, 
\begin{equation}\label{eq:system}
\begin{aligned}
\underset{\pi}{\text{max}} ~~~~~~ &\sum_{t=1}^\infty \sum_{j, a} \gamma^{t-1}  P^\pi(s_t = j, a_t = a) \mathcal{R}(j, a), \\
& \sum_{t=1}^\infty  \sum_{j, a} \gamma^{t-1} P^\pi(s_t = j, a_t = a) \mathcal{C}(j, a) \leq E.
\end{aligned}
\end{equation}

The  canonical  solution  methodology  for \eqref{eq:system} is based  on  convex linear  programming \cite{altman1999constrained}. Of interest is the following dual linear program\footnote{The primal linear program is defined over a convex set of occupancy measures, and by the Lagrangian strong duality, \eqref{eq:system} can be obtained. Interested readers can find more details about this formulation in Chapter 3 \cite{altman1999constrained}.} which can be solved for the optimal state-value function and Lagrange penalty multiplier $\mu$,
\begin{equation}\label{eq:primal2}
\begin{aligned}
\underset{V(i), \mu \geq 0}{\text{min}} ~~~~~~ &\sum_{\forall i} \beta(i)V(i) + \mu E, \\
&V(i) \geq \mathcal{R}(i,a) - \mu \mathcal{C}(i,a) + \gamma \sum_j P_{ij}(a)V(j), \\
&\forall (i ,a) \in \mathcal{S} \times \mathcal{A}.
\end{aligned}
\end{equation}
Note that for a fixed Lagrangian penalty multiplier 
$\mu$, \eqref{eq:primal2} is analogous to \eqref{eq:primal1}, which can be solved using dynamic programming and Bellman's optimality equation \cite{altman1999constrained},
\begin{equation} \label{opt2}
\resizebox{1\hsize}{!}{$
V_{\pi^*} (i,\mu) =\underset{a \in \mathcal{A}(i)}{\text{max}} [ \mathcal{R}(i, a) - \mu \mathcal{C}(i,a) + \gamma \sum_{j \in \mathcal{S}} P_{ij}(a) V_{\pi^*} (j,\mu) ] , \forall i \in \mathcal{S}.$}
\end{equation}

This formulation has led to the development of multi time-scale stochastic approximation based Lagrangian primal-dual type learning algorithms to solve CMDPs with and without value function approximation \cite{geibel2005risk,geibel2012learning,borkar2005actor,abad2002self,bhatnagar2010actor,bhatnagar2012online,prashanth2014actor,tessler2018reward,liang2018accelerated,fu2018risk,achiam2017constrained,raybenchmarking,ding2020provably,paternain2019safe,khairy2020constrained}. In such algorithms, $\mu$ is updated by gradient descent at a slow time-scale, while the state-value function is optimized on a faster time-scale using dynamic programming (model-based) or gradient-ascent (model-free), until an optimal saddle point $(\vect{V}_{\pi^*}(\mu^*),\mu^*)$ is reached. These approaches however suffer from the typical problems of gradient optimization on non-smooth objectives, sensitivity to the initialization of $\mu$, and convergence dependency on the learning rate sequence used to update $\mu$ \cite{chow2017risk, achiam2017constrained,liu2019ipo, liang2018accelerated}. As we will show in the next section, the optimization objective in \eqref{eq:primal2} is piecewise linear convex with respect to $\mu$, and so it is not differentiable at the cusps. This means that we cannot assess the optimality of a solution because $\mu^*$ usually does not have a zero gradient. This also means that convergence to $\mu^*$ is dependent on specifying an optimal learning rate sequence, which is by itself a very challenging task. Unlike existing methods, our proposed algorithm explicitly tackles these problems by exploiting problem structure, and so it does not require hyper-parameter tuning or specifying an optimal learning rate sequence.

\section{A Novel CMDP Framework}
In this section, we describe our novel methodology to solve finite CMDPs with discrete state-action space and known probability transition function. For ease of notation, we hereby consider the case of one constraint, yet our proposed algorithm can be readily extended to the multi-constraint case by successively iterating the method described hereafter over each Lagrange penalty multiplier, one at a time. 

Our proposed method works in an iterative two-level optimization scheme.  For a given Lagrange penalty multiplier $\mu$, an unconstrained MDP with a penalized reward function $\hat{\mathcal{R}}(i,a)=\mathcal{R}(i, a) - \mu \mathcal{C}(i,a), \forall i,a$, is specified, and we require the corresponding optimal value function $V_{\pi^*}(i,\mu), \forall i$ to be found using dynamic programming \eqref{opt2}. Denote the optimization objective in \eqref{eq:primal2} as a function of $\mu$ by $\mathcal{O}(\mu)$, i.e., $\mathcal{O}(\mu) = \sum_{\forall i} \beta(i)V_{\pi^*}(i,\mu) + \mu E$. Thus to evaluate $\mathcal{O}(\mu)$, one has to solve for $V_{\pi^*}(i,\mu), \forall i$ first. How to efficiently search for the optimal $\mu^*$? To answer this question, we first prove that $\mathcal{O}(\mu)$ is a piecewise linear convex function with respect to $\mu$, and design an efficient solution algorithm next.  
\begin{definition}
Let $a_1^i, \cdots, a_m^i \in \mathbb{R}$, and $b_1^i, \cdots, b_m^i \in \mathbb{R}$ for some positive integer $i \in \mathbb{Z}^+$. A Piecewise Linear\footnote{Affine is more accurate but less common.} Convex function (PWLC) is given by $f_i(\mu)=\text{max}_{l=1,\cdots,m}~\big\{a_l^i \mu + b_l^i \big\}$.
\end{definition}
\begin{theorem}
Given that $V_{\pi^*}(i,\mu), \forall i$ is the optimal value function
of an unconstrained MDP with a penalized reward function $\hat{\mathcal{R}}(i,a)=\mathcal{R}(i, a) - \mu \mathcal{C}(i,a), \forall i,a$, $\beta(i)$ is the probability that the initial state is $i$, and $E$ is a constant representing the cost constraint upper bound in the original CMDP, $\mathcal{O}(\mu) = \sum_{\forall i} \beta(i)V_{\pi^*}(i,\mu) + \mu E$ is a PWLC function with respect to $\mu$.
\end{theorem}
\begin{proof}
For ease of exposition, we introduce the following three properties of PWLC functions, and show how they can be used to construct the proof, 
\begin{enumerate}
\item  $f_i(\mu)+f_j(\mu)$ \textit{is PWLC for any} $i,j \in \mathbb{Z}^+$. Notice that for point-wise maximum,  $\text{max}_l \{a_l\} + \text{max}_k \{a_k\} = \text{max}_{l,k}\{a_l+a_k\}$. Hence, 
$f_i(\mu)+f_j(\mu)$ $=\text{max}_{l,k}~\big\{ (a_l^i+a_k^j) \mu + (b_l^i+b_k^j) \big\}$, which is PWLC.
\item $\text{max}_l~\big\{a_l^i \mu + b_l^i + f_j(\mu) \big\}$ \textit{is PWLC for any} $i,j \in \mathbb{Z}^+$. Notice that $\text{max}_l \Big\{ a_l + \text{max}_k \{a_k\} \Big\}$  
$=\text{max}_l \{a_l\} + \text{max}_k \{a_k\}$
$= \text{max}_{l,k}\{a_l+a_k\}$. Hence,
$\text{max}_l~\big\{a_l^i \mu + b_l^i + f_j(\mu) \big\}$ $=\text{max}_l~\Big\{a_l^i \mu + b_l^i + \text{max}_k~\big\{a_k^j \mu + b_k^j \} \Big\}$ $=\text{max}_{l,k}~\big\{ (a_l^i+a_k^j) \mu + (b_l^i+b_k^j) \big\}$, which is PWLC.
\item $f_i(\mu) + c\mu$ \textit{is PWLC}. Notice that $f_i(\mu) + c\mu = \text{max}_l~\big\{a_l^i \mu + b_l^i \big\}+c\mu$ $=\text{max}_l~\big\{(a_l^i+c) \mu + b_l^i \big\}$, which is PWLC.
\end{enumerate}
Based on these properties, the proof proceeds as follows. Given $V_{\pi^*} (i,\mu)$, the recursive formula in \eqref{opt2} is expanded by iteratively substituting for $V_{\pi^*} (i,\mu)$ with the right hand side of \eqref{opt2}. This expansion allows us to deduce that $V_{\pi^*} (i,\mu)$ has the following functional form, 
\begin{equation} \nonumber
\resizebox{1\hsize}{!}{$
V_{\pi^*} (i,\mu) =\text{max}_l~\Bigg\{a_l^i \mu + b_l^i + \gamma \text{max}_k~\Big\{ a_k^{i^\prime} \mu + b_k^{i^\prime}  + \gamma \text{max}_m~\big\{ a_m^{i^{\prime \prime }} \mu + b_m^{i^{\prime \prime}} + \cdots \big\} \Big\} \Bigg\} $},
\end{equation}
for some constants $a_l^i, b_l^i,a_k^{i^\prime},b_k^{i^\prime}, \cdots \in \mathbb{R}$ which come from discounted rewards, costs, and their product with the transition probability function. Based on successive application of properties 1 and 2, $V_{\pi^*} (i,\mu), \forall i$ are PWLC functions with respect to $\mu$. Also,  $\sum_{\forall i} \beta(i)V_{\pi^*} (i,\mu)$ is PWLC with respect to $\mu$ based on property 1. Finally, $\mathcal{O}(\mu) = \sum_{\forall i} \beta(i)V_{\pi^*}(i,\mu) + \mu E$ is a PWLC function with respect to $\mu$ based on property 3. 
\end{proof}

To efficiently find $\mu^*$ and solve \eqref{eq:primal2}, we propose a novel Gradient-Aware Search (GAS) algorithm which exploits the piecewise linear convex structure. Our proposed GAS algorithm, which is outlined in Algorithm 1, operates over two-loops. For a given Lagrange multiplier $\mu^\times$ (line 6) at an outer loop iteration, the inner loop finds the optimal state-value function, $V_{\pi^*}(i,\mu^\times), \forall i$ using value iteration, as well as the gradient components of $\mathcal{O}(\mu)$ with respect to  $\mu^\times$ (lines 7-16). Based on this information, the outer loop exploits the PWLC structure to calculate the next $\mu^\times$ (line 17, lines 20-25). This process continues until a termination criteria is met (lines 18-19). In what follows, the proposed GAS algorithm is outlined in detail. 

\begin{algorithm} \label{Alg1}
    \SetKwInOut{Input}{Input}
    \SetKwInOut{Output}{Output}
    \SetKwInOut{AlgP}{Algorithm Parameters}
    \Input{$\mathcal{R}(i,a),~\mathcal{C}(i,a),~P_{ij}(a)$, $\forall (i,a) \in \mathcal{S} \times \mathcal{A}$ }
    \Output{$V^*(i),~\forall i \in \mathcal{S}$, $\mu^*$}
    \AlgP{$\epsilon$, $\epsilon^\prime$}
    \textbf{Initialize:} $\mu^+=M$, $\mu^-=0$, $\mathcal{O}(\mu^+)$, $\mathcal{O}(\mu^-)$
    $\frac{\partial \mathcal{O}}{\partial \mu^+}$,$\frac{\partial \mathcal{O}}{\partial \mu^-},  \mathcal{O}_{\text{min}}=\infty$\\ 
\Repeat{forever}{    
    $\Delta=\infty$, $k=0$, $V_k(i) = 0,~\forall i \in \mathcal{S}$,\\
    $\omega_k^1(i,a) =0,~\forall (i,a) \in \mathcal{S} \times \mathcal{A} $\\
    $Q_k(i,a) =0,~\forall (i,a) \in \mathcal{S} \times \mathcal{A} $ \\
    Find the intersection $\mu^\times$ of the two lines defined by ($\mathcal{O}(\mu^+)$, $\frac{\partial \mathcal{O}}{\partial \mu^+}$) and ($\mathcal{O}(\mu^-)$, $\frac{\partial \mathcal{O}}{\partial \mu^-}$)\\
\Repeat{$\Delta < \epsilon$}
{
$Q_k(i,a,\mu^\times) = \mathcal{R} (i, a)- \mu^\times \mathcal{C}(i,a)$ \\
$~~~~~~~~~~~~~~+\gamma \sum_{j \in \mathcal{S}} P_{ij}(a) V_k (j)  ,~\forall (i,a) \in \mathcal{S} \times \mathcal{A}$
\\
$V_{k+1} (i,\mu^\times) =\underset{a \in \mathcal{A}(i)}{\text{max}} Q_k(i,a,\mu^\times), ~\forall i \in \mathcal{S}$\\
$ \tilde{a}(i) = \underset{a \in \mathcal{A}(i)}{\text{argmax}} ~Q(i,a,\mu^\times), ~\forall i \in \mathcal{S}$\\
$\omega_{k+1}^1(i,a) =\mathcal{C}(i,a)+\zeta \sum_{j \in \mathcal{S}} P_{ij}(a) \omega_{k}^1(j,\tilde{a}(j)) ,$
\\$~~~~~~~~~~~~~~~~~~~~~~~~~~~~~~~~~~~~~~~\forall (i,a) \in \mathcal{S} \times \mathcal{A} $\\  
$\Delta = \frac{1}{|\mathcal{S}|}\sum_i \frac{|V_{k+1}(i,\mu^\times) - V_k(i,\mu^\times)|}{|V_{k}(i,\mu^\times)|}$\\
$k = k + 1$\\
}
Calculate $\mathcal{O}(\mu^\times)$, $\frac{\partial \mathcal{O}}{\partial \mu^\times}$ based on \eqref{lagrange1} and \eqref{lagrange2}\\
\uIf{  $|\mathcal{O}_{\text{min}} -  \mathcal{O}(\mu^\times)|  \leq \epsilon^\prime $}
{
\textbf{Break}
}
\uIf{  $\mathcal{O}_{\text{min}} >  \mathcal{O}(\mu^\times)$}
{
$\mathcal{O}_{\text{min}} =  \mathcal{O}(\mu^\times)$
}
\uIf{  $0 \leq \frac{\partial \mathcal{O}}{\partial \mu^\times} < \frac{\partial \mathcal{O}}{\partial \mu^+}$ }
{ 
$\mu^+=\mu^\times$, $\frac{\partial \mathcal{O}}{\partial \mu^+} =\frac{\partial \mathcal{O}}{\partial \mu^\times}$
}

\uIf{  $\frac{\partial \mathcal{O}}{\partial \mu^-} < \frac{\partial \mathcal{O}}{\partial \mu^\times} < 0$ }
{ 
$\mu^-=\mu^\times$, $\frac{\partial \mathcal{O}}{\partial \mu^-} =\frac{\partial \mathcal{O}}{\partial \mu^\times}$
}
}
$\mu^* = \mu^+$ \\
$V^*(i) = V_{k+1}(i),~\forall i \in \mathcal{S}$
    \caption{Gradient-Aware Search CMDP Solver}
\end{algorithm}

\textbf{Inner loop:} The value iterations of the inner loop are sub-scripted with the index $k$. Let $Q(i,a,\mu), \forall i,a$, denote the state-action value function  which is the expected discounted return starting in state $i$, taking action $a$, and following $\pi$ thereafter, for a fixed $\mu$. Note that we drop the subscript $\pi$ to avoid notation clutter, but dependence on $\pi$ remains unchanged. The state-action value function at each step $k$ can be represented as a linear function of $\mu^\times$, 
\begin{equation}
   Q_k(i,a,\mu^\times) = \omega^0_k(i,a) - \omega^1_k(i,a) \mu^\times.
\end{equation}
In this representation, accumulated discounted rewards contribute to $\omega^0_k(i,a)$ and accumulated discounted costs contribute to $\omega^1_k(i,a)$ (line 12). Given $\omega^0_k(i,a),~\omega^1_k(i,a),$ $\forall a$, and $\mu^\times$, the value function at state $i$ can be obtained by $V_k(i, \mu^\times) = \underset{a}{\text{max}}~Q_k(i,a, \mu^\times) =\omega^0_k(i,\tilde{a}) - \omega^1_k(i,\tilde{a}) \mu^\times $ (line 10), where $\tilde{a}$ denotes the action which attains the maximum state-value among the set of actions at state $i$, $\tilde{a}(i) = \underset{a}{\text{argmax}}~Q_k(i,a,\mu^\times)$ (line 11). Notice that in the inner loop, two functions are estimated using value iterations, the total state-action value function (line 8), and the state-action cost function (line 12). Dynamic programming steps in the inner loop are performed until the mean relative absolute error in the estimation of the state value function is less than a preset accuracy threshold $\epsilon$ (line 14-16). When the inner loop concludes, the optimal $V_{\pi^*}(i, \mu^\times),~ \omega^1_{\pi^*}(i,\tilde{a}), \forall i$, are returned to the outer loop. 

\textbf{Outer loop:} Given $\mu^\times$, and $V_{\pi^*}(i, \mu^\times),~ \omega^1_{\pi^*}(i,\tilde{a}), \forall i$ from the inner loop, $\mathcal{O}(\mu^\times)$ and the gradient of $\mathcal{O}(\mu^\times)$ with respect to $\mu^\times$, can be obtained based on (line 17), 
\begin{equation} \label{lagrange1}
\begin{aligned}
&\mathcal{O}(\mu^\times) = \sum_i V_{\pi^*}(i, \mu^\times)  \beta(i) + \mu^\times E,
\end{aligned}
\end{equation}
and,
\begin{equation} \label{lagrange2}
\begin{aligned}
\frac{\partial \mathcal{O}}{\partial \mu^\times} &= \frac{\partial}{\partial \mu^\times} \sum_i  V_{\pi^*}(i, \mu^\times)  \beta(i) + \mu^\times E
\\
&= \frac{\partial}{\partial \mu^\times} \sum_i  \underset{a}{\text{max}}~Q_{\pi^*}(i,a,\mu^\times)  \beta(i) + \mu^\times E
\\
&= \frac{\partial}{\partial \mu^\times} \sum_i \big( \omega^0_k(i,\tilde{a}) - \omega^1_k(i,\tilde{a}) \mu^\times \big)  \beta(i) + \mu^\times E
\\
&=  -\sum_i \omega^1_{\pi^*}(i,\tilde{a})  \beta(i) + E.
\end{aligned}
\end{equation}

When $\frac{\partial \mathcal{O}}{\partial \mu^\times}$ is negative, i.e., $\sum_i \omega^1_{\pi^*}(i,\tilde{a})  \beta(j) >  E$, the current policy $\pi^*(\mu^\times)$ is infeasible, which means $\mu^\times$ should be increased. On the other hand, when $\frac{\partial \mathcal{O}}{\partial \mu^\times}$ is non-negative, i.e., $\sum_i \omega^1_{\pi^*}(i,\tilde{a})  \beta(j) \leq  E$, the current policy is feasible, yet we can potentially obtain another feasible policy with larger returns by decreasing $\mu^\times$. Notice that if $\mu^\times$ is a cusp which lies at the intersection point of two linear segments of the PWLC function $\mathcal{O}(\mu)$,  $\frac{\partial \mathcal{O}}{\partial \mu^\times}$ does not exist in theory, although the two one-sided derivatives (left and right) are well-defined. The nonexistence of a derivative at cusps does not pose a challenge to our proposed algorithm, because computations are accurate up to the machine's floating point precision. To further clarify, suppose $\mu^\times$ is a cusp on the curve of $\mathcal{O}(\mu)$. In practice, we can only compute the  gradient \eqref{lagrange2} at $(\mu^\times \pm \hat{\epsilon})$, where $\hat{\epsilon}$ is a very small precision loss due to rounding in floating point arithmetic, which is at the order of $10^{-16}$ for double precision as defined by the IEEE 754-2008 standard. This means that computing the gradient at a point is essentially computing a one-sided derivative, which always exists.

It is worth to mention that our proposed algorithm does not perform gradient optimization (which suffers on non-smooth objectives) but rather exploits the structure of the PWLC objective to find the optimal $\mu^*$. The algorithm always retains two values for $\mu$, $\{\mu^-, \mu^+\}$, where $\frac{\partial \mathcal{O}}{\partial \mu^-} <0$ and $\frac{\partial \mathcal{O}}{\partial \mu^+} \geq 0$. $\mu^-$ along with   $\frac{\partial \mathcal{O}}{\partial \mu^-}$ define a line $L^-$, while $\mu^+$ along with   $\frac{\partial \mathcal{O}}{\partial \mu^+}$ define another line $L^+$. $L^-$ and $L^+$ intersect at a point $\mu^\times$, which is passed to the inner loop. There are two possible outcomes when running the inner loop with $\mu^\times$, 
\begin{enumerate}
    \item $0 \leq \frac{\partial \mathcal{O}}{\partial \mu^\times} < \frac{\partial \mathcal{O}}{\partial \mu^+}$, in this case $\mathcal{O}(\mu^\times) < \mathcal{O}(\mu^+)$ because $\mathcal{O}(\mu)$ is PWLC with respect to $\mu$. The algorithm retains $\mu^\times$ and $\frac{\partial \mathcal{O}}{\partial \mu^\times}$ by setting, $\mu^+=\mu^\times$ and $\frac{\partial \mathcal{O}}{\partial \mu^+} =\frac{\partial \mathcal{O}}{\partial \mu^\times}$ (lines 22-23).
    
    \item  $ \frac{\partial \mathcal{O}}{\partial \mu^-} < \frac{\partial \mathcal{O}}{\partial \mu^\times} < 0$, in this case $\mathcal{O}(\mu^\times) < \mathcal{O}(\mu^-)$ because $\mathcal{O}(\mu)$ is PWLC with respect to $\mu$. The algorithm retains $\mu^\times$ and $\frac{\partial \mathcal{O}}{\partial \mu^\times}$ by setting, $\mu^-=\mu^\times$ and $\frac{\partial \mathcal{O}}{\partial \mu^-} =\frac{\partial \mathcal{O}}{\partial \mu^\times}$ (lines 24-25).
\end{enumerate}
The algorithm terminates when the absolute error between  $\mathcal{O}_\text{min}$ which is the minimum objective value found so far, and $\mathcal{O}(\mu^\times)$, is less than an arbitrarily small positive number $\epsilon^\prime$ (lines 18-19). 
Note that the initial $\mu^+$ is set to $M$, an arbitrarily large number which attains $\frac{\partial \mathcal{O}}{\partial \mu^+} \geq 0$. If no such $M$ can be found, then \eqref{eq:primal2} is not feasible, i.e., $\mathcal{O}(\mu)$ is not lower bounded and can be made arbitrarily small $\mathcal{O}(\mu) \rightarrow -\infty$ by taking $\mu  \rightarrow \infty$. 
Figure \ref{illustration} below shows a visual illustration of how Algorithm 1 works in practice. It is worth to mention that Algorithm 1 only evaluates $\mathcal{O}(\mu)$ at the points indicated by a green circle.
\begin{figure}[h]  
    \centering
    \includegraphics[width=0.30\textwidth]{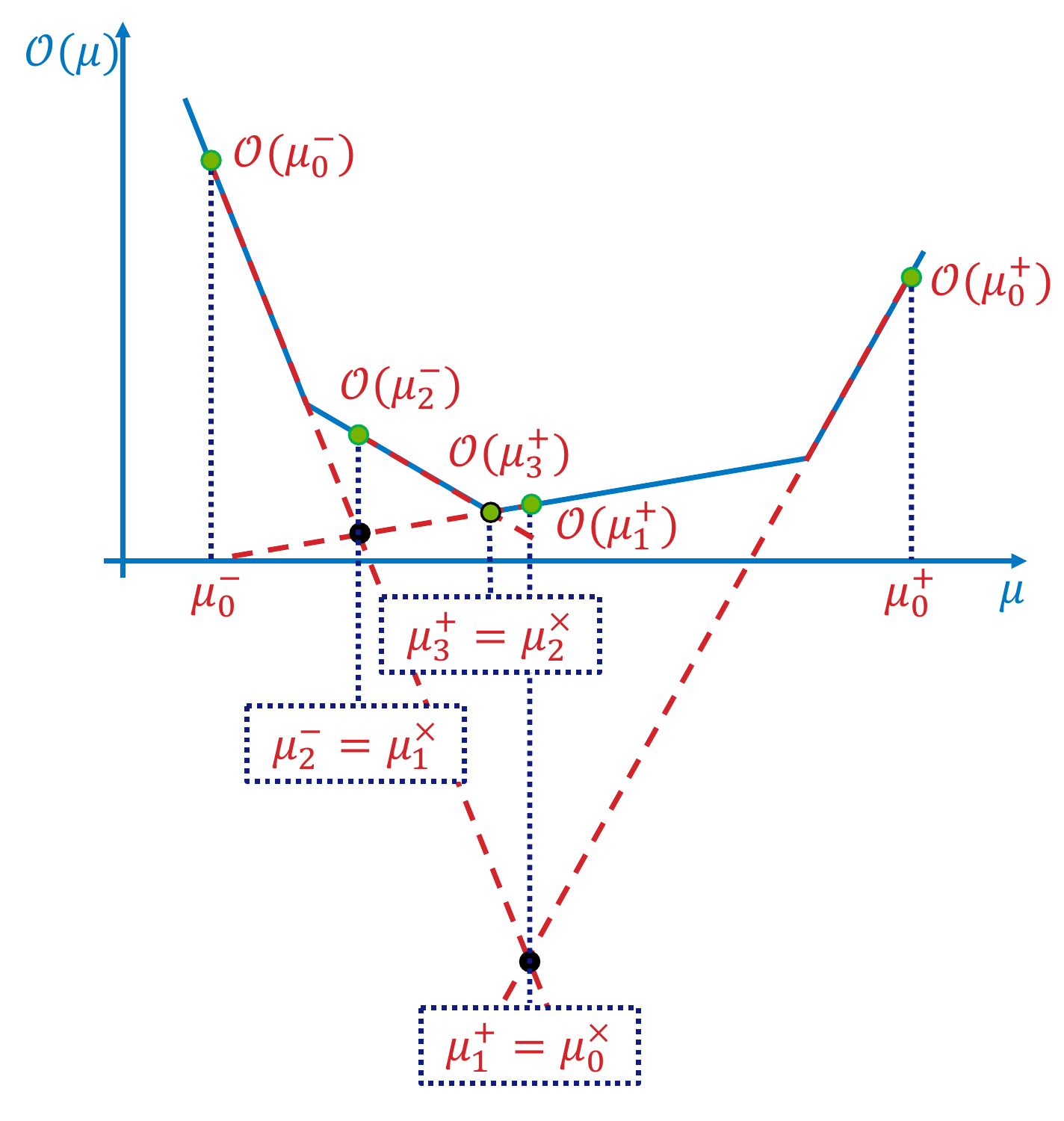}
     \caption{Visual illustration of Algorithm 1. 
     The proposed algorithm always retains two values for the Lagrange multiplier $\mu$, $\{\mu^-, \mu^+\}$, where $\frac{\partial \mathcal{O}}{\partial \mu^-} <0$ and $\frac{\partial \mathcal{O}}{\partial \mu^+} \geq 0$. $\mu^-$ along with   $\frac{\partial \mathcal{O}}{\partial \mu^-}$ define a line $L^-$, while $\mu^+$ along with   $\frac{\partial \mathcal{O}}{\partial \mu^+}$ define another line $L^+$. $L^-$ and $L^+$ intersect at a point $\mu^\times$, at which $\mathcal{O}(\mu^\times)$ is evaluated (green circle). Based on this evaluation, $\mu^-, \mu^+$ are updated and the process is repeated until the global minima is reached. In this example, $\mu_3^+=\mu_2^\times$ is the optimal $\mu^*$.}
     \label{illustration}
\end{figure}

\section{Case Studies}

In this section, we evaluate the efficacy of the proposed GAS algorithm 
in two different domains, robot navigation in a grid world, and solar-powered UAV-based wireless network management. In all the experiments, $\epsilon$ and $\epsilon^\prime$ are set to  $10^{-10}$ unless otherwise specified.

\subsection{Grid World Robot Navigation}

In the grid world robot navigation problem, the agent controls the movement of a robot in a rectangular grid world, where states represent grid points on a 2D terrain map. The robot starts in a safe region in the bottom right corner (blue square), and is required to travel to the goal located in the top right corner (green square), as can be seen from Figure \ref{gridvalue}. At each time step, the agent can move in either of four directions to one of the neighbouring states. However, the movement direction of the agent is stochastic and partially depends on the chosen direction. Specifically, with probability $1-\delta$, the robot will move in the chosen direction, and uniformly randomly otherwise, i.e., with probability $\delta/4$ the robot will move to one of the four neighboring states. At each time step, the agent receives a reward of $-1$ to account for fuel usage. Upon reaching the goal, the agent receives a  positive reward of $\hat{M} >> 1$. In between the starting and destination states, there are a number of obstacles (red squares) that the agent should avoid. Hitting an obstacle costs the agent $\hat{M}$. It is important to note that the 2D terrain map is built such that a shorter path induces higher risk of hitting obstacles. By maximizing long-term expected discounted rewards subject to long-term expected discounted costs, the agent finds the shortest path from the starting state to the destination state such that the toll for hitting obstacles does not exceed an upper bound. This problem is inspired by classic grid world problems in the literature \cite{tessler2018reward, chow2015risk}.

\subsubsection{Experiment Results} We choose a grid of $20 \times 20$ with a total of $400$ states. The start state is $(2,18)$, the destination is $(19,18)$,
$\hat{M}=2/(1-\gamma)$, $\delta=0.05$, $\gamma=0.99$, and $30$ obstacles are deployed. We follow \cite{chow2015risk} in the choice of parameters, which trades off high penalty for obstacle collisions and computational complexity. $\mathcal{O}(\mu)$ is plotted in Figure \ref{gridresults}(a) over $\mu \in [0,200]$. 
It can be seen from Figure \ref{gridresults}(a) that $\mathcal{O}(\mu)$ is a PWLC function as given by Theorem $1$. It can be also seen that the global minima is a non-differentiable cusp around $90$ with unequal one-sided derivatives. 

In Figure  \ref{gridresults}(b), we compare the convergence performance of the proposed GAS with binary search (BS), by plotting the number of outer loop iterations versus an accuracy level $\epsilon^\prime$. The initial search window for $\mu$ is $[0,M]$\footnote{$M$ is problem specific. The ability to choose an arbitrarily high value for $M$ without negatively impacting convergence time is desirable because it demonstrates the algorithm's capability to make adaptively large strides in the search space towards the optimal solution.}, and $M$ is either $10^3$ or $10^5$. Given a value for $\mu$, the inner loop evaluates \eqref{lagrange1} and its gradient \eqref{lagrange2}. The outer loop then determines the next $\mu$ either based on the proposed GAS or BS. This iterative procedures continues until the convergence criteria is met. Compared to BS, the proposed GAS algorithm requires a smaller number of outer loop iterations for all levels of accuracy thresholds $\epsilon^\prime$. This is because GAS adaptively shrinks the search window by evaluating points at the intersection of line segments with non-negative and negative gradients, whereas BS blindly shrinks the window size by $\frac{1}{2}$ every iteration. These results demonstrate that $\epsilon^\prime$ can be set arbitrarily small and $M$ can be set arbitrarily large without substantially impacting the convergence time. 

In Figure \ref{gridresults}(c) we compare the convergence performance of the proposed GAS with the Lagrangian primal-dual optimization (Lagrangian PDO), by plotting the total number of value iterations, which is the sum of the number of value iterations over all outer loop iterations, versus the learning rate decay parameter in a log-log scale. In the Lagrangian PDO, gradient descent on $\mu$ is performed along side dynamic programming iterations on $V(i),\forall i$, which motivates us to compare the convergence performance in terms of total number of value iterations. In Lagrangian PDO, the update rule for $\mu$ is $\mu_{k+1} = \mu_k - \kappa_k \frac{\partial \mathcal{O}}{\partial \mu_k}$, where $\kappa_k$ is the step size parameter. In our experiments, $\kappa_0$ is initially set to $1$ to speed up descending towards the minima, and is decayed according to $\kappa_{k+1} = \kappa_k e^{-\xi T}$, where $\xi$ is the learning rate decay parameter, and $T$ is the number of times $\frac{\partial \mathcal{O}}{\partial \mu_k}$ changes signs. While such choice of learning rate decay does not necessarily produce an optimal sequence, it serves to expose the sensitivity of gradient based optimization to the initial $\mu_0$ and decay parameter $\xi$. 
For every $\xi \in [10^{-4},1]$, $\mu_0$ is uniformly randomly initialized $100$ times from $[0,M]$, and the mean number of total value iterations is plotted in Figure \ref{gridresults}(c). It can be seen from Figure \ref{gridresults}(c) that the convergence of the Lagrangian PDO method is sensitive to the initialization of $\mu_0$ which can be inferred from the initial window size, and the decay rate parameter. For some $\xi$, Lagrangian PDO converges faster than GAS, but for other values it lags behind. In practice, $\xi$ is a hyper-parameter which requires tuning, consuming extra computational resources. In contrast, the proposed GAS algorithm works out of the box without the need to specify a learning rate sequence or tune any hyper-parameters. 

In Table \ref{LP}, we compare the results obtained by the proposed GAS algorithm at convergence with those obtained by solving \eqref{eq:primal2} using Linear Programming. We use the LP solver provided by Gurobi, which is arguably the most powerful mathematical optimization solver in the market. The optimal $\mu^*$, $\mathcal{O}(\mu^*)$, $\text{min}_i|BE(i)|$, $\mathbb{E}_i[|BE(i)|]$, and $\text{max}_i|BE(i)|$ are reported, where $BE(i)$ is the Bellman error at state $i$ given by, $BE(i) = V_{\pi^*} (i,\mu^*) -\underset{a \in \mathcal{A}(i)}{\text{max}} [ \mathcal{R}(i, a) - \mu \mathcal{C}(i,a) + \gamma \sum_{j \in \mathcal{S}} P_{ij}(a) V_{\pi^*} (j,\mu^*) ] , \forall i \in \mathcal{S}$. It can be observed that Gurobi's LP solver converges to a sub-optimal solution with a higher objective value compared with the proposed GAS algorithm (recall that \eqref{eq:primal2} is a minimization problem, hence the lower $\mathcal{O}(\mu)$ the better). This demonstrates that generic LP solvers may struggle to solve CMDPs, which has motivated the research community to develop CMDP specific algorithms.

\begin{table}[!htbp]
\centering
\caption{Performance comparison with LP}
\begin{tabular}{c | c | c| c|c}
\toprule
 &  \multicolumn{2}{c}{Grid World} & \multicolumn{2}{c}{Solar-Powered UAV-}\\
  &  \multicolumn{2}{c}{Robot Navigation} & \multicolumn{2}{c}{Based Wireless Network}\\
\midrule
{}   & LP (Gurobi)   & GAS   & LP (Gurobi)  & GAS \\
\hline
$\mu^*$  &  68.33333 & 90.08102 & 0.030786  & 0.030786\\
$\text{min}_i|BE(i)|$   &  0.0 & 3.11e-07  & 0.0  & 9.32e-09\\
$\mathbb{E}_i[|BE(i)|]$    &  9.89e-10 &  3.11e-07   & 0.001788  & 9.32e-09\\
$\text{max}_i|BE(i)|$    &  1.13e-07  &  3.11e-07   & 1.160397  & 9.33e-09\\
$\mathcal{O}(\mu^*)$    &  15216.83  &  15184.57  & 93.92830  & 93.92830\\
\bottomrule
\end{tabular} \label{LP}
\end{table}
\begin{figure*}[h]  
    \begin{minipage}[b]{.33\linewidth} 
     \includegraphics[width=1\textwidth]{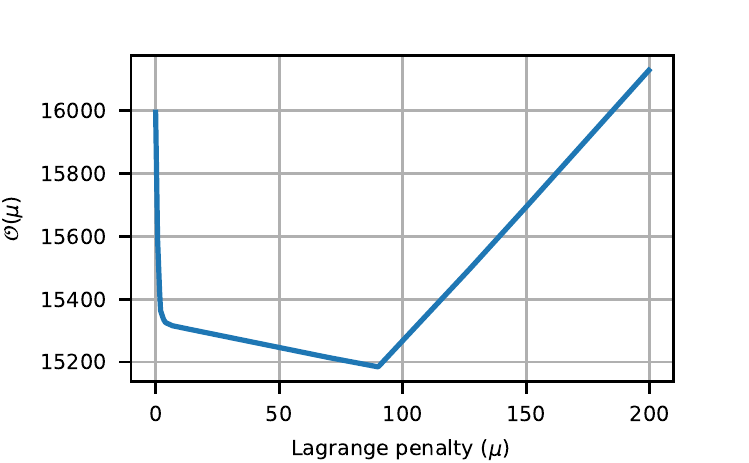}
     \subcaption{The piecewise linear convex function $\mathcal{O}(\mu)$}
   \end{minipage} 
   \begin{minipage}[b]{.33\linewidth}
     \includegraphics[width=1\textwidth]{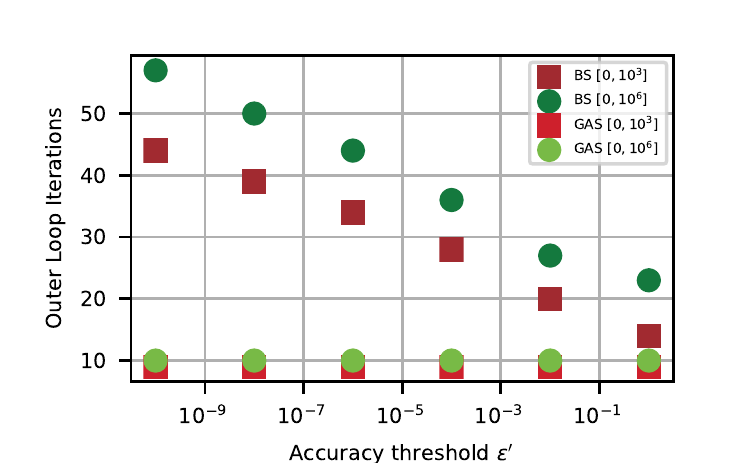}
     \subcaption{Comparison with Binary Search (BS)}
   \end{minipage} 
   \begin{minipage}[b]{.33\linewidth}
     \includegraphics[width=1\textwidth]{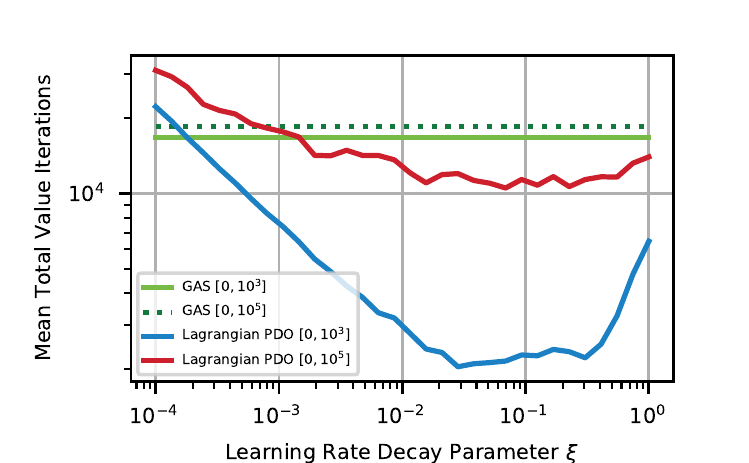}
     \subcaption{Comparison with Lagrangian PDO}
   \end{minipage}
 \caption{Performance comparison on the grid world robot navigation problem}
 \label{gridresults}
\end{figure*}

In Figure \ref{gridvalue}, filled contours maps for the value function are plotted. Regions which have the same value have the same color shade. The value function is plotted for four different values of the cost constraint upper bound, along with the start state (blue square), destination (green square), and obstacles (red squares). With a tighter cost constraint $(E=5)$, a safer policy is learned in which the agent takes a long path to go around the obstacles as in Figure \ref{gridvalue}(d), successfully reaching the destination $P_{\pi^*}^s=99.90\%$ of the times. When the cost constraint is loose $(E=160)$, a riskier policy is learned in which the agent takes a shorter path by navigating between the obstacles as in Figure \ref{gridvalue}(a), successfully reaching the destination  $P_{\pi^*}^s=76.85\%$ of the times, based on $2000$ roll outs. 
\begin{figure*}[ht]
   \begin{minipage}[b]{.245\linewidth}
     \centering
     \includegraphics[width=1\textwidth]{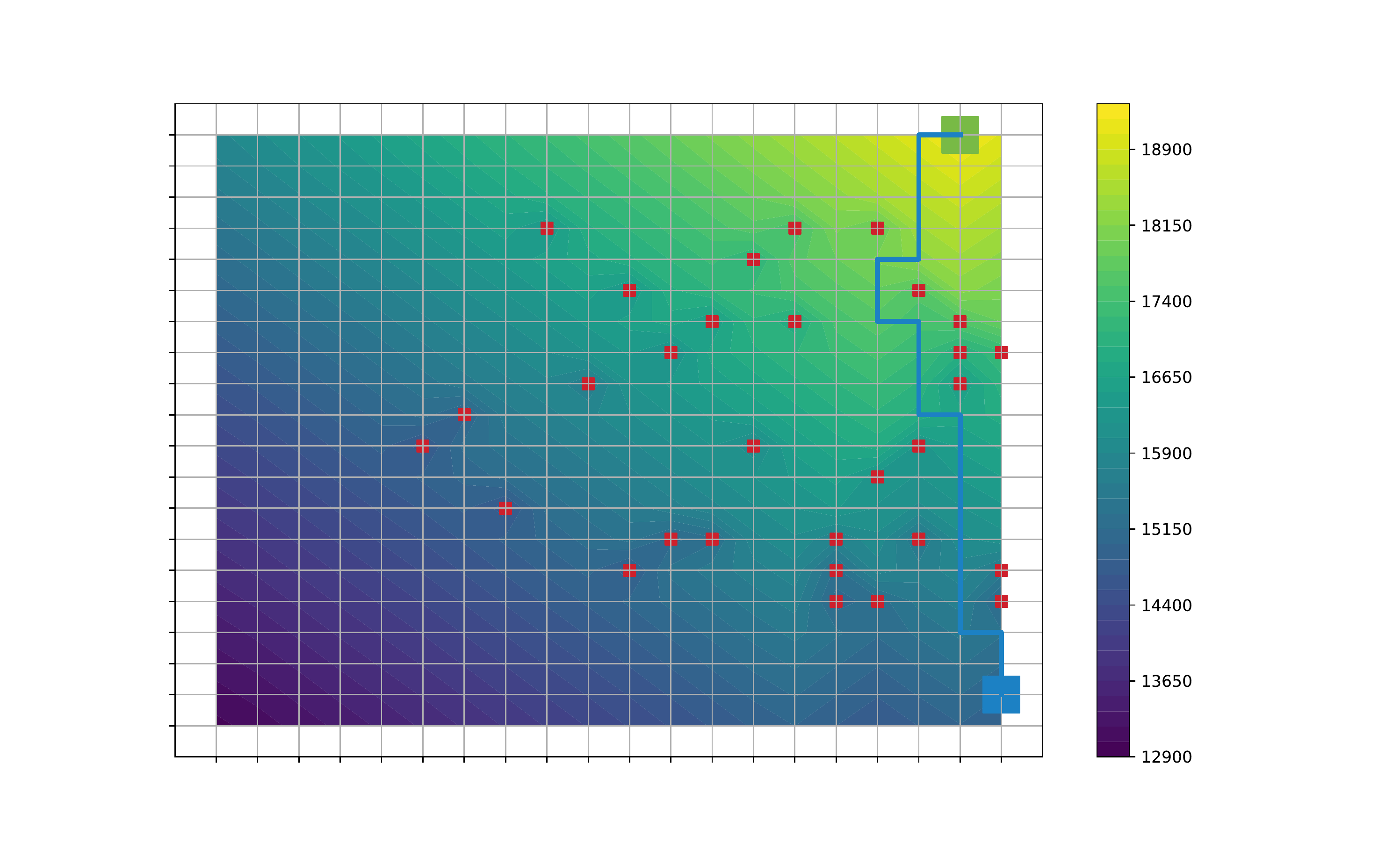}
     \subcaption{$P_{\pi^*}^s=76.85\%, E =160$}
   \end{minipage}
   \begin{minipage}[b]{.245\linewidth}
     \centering
     \includegraphics[width=1\textwidth]{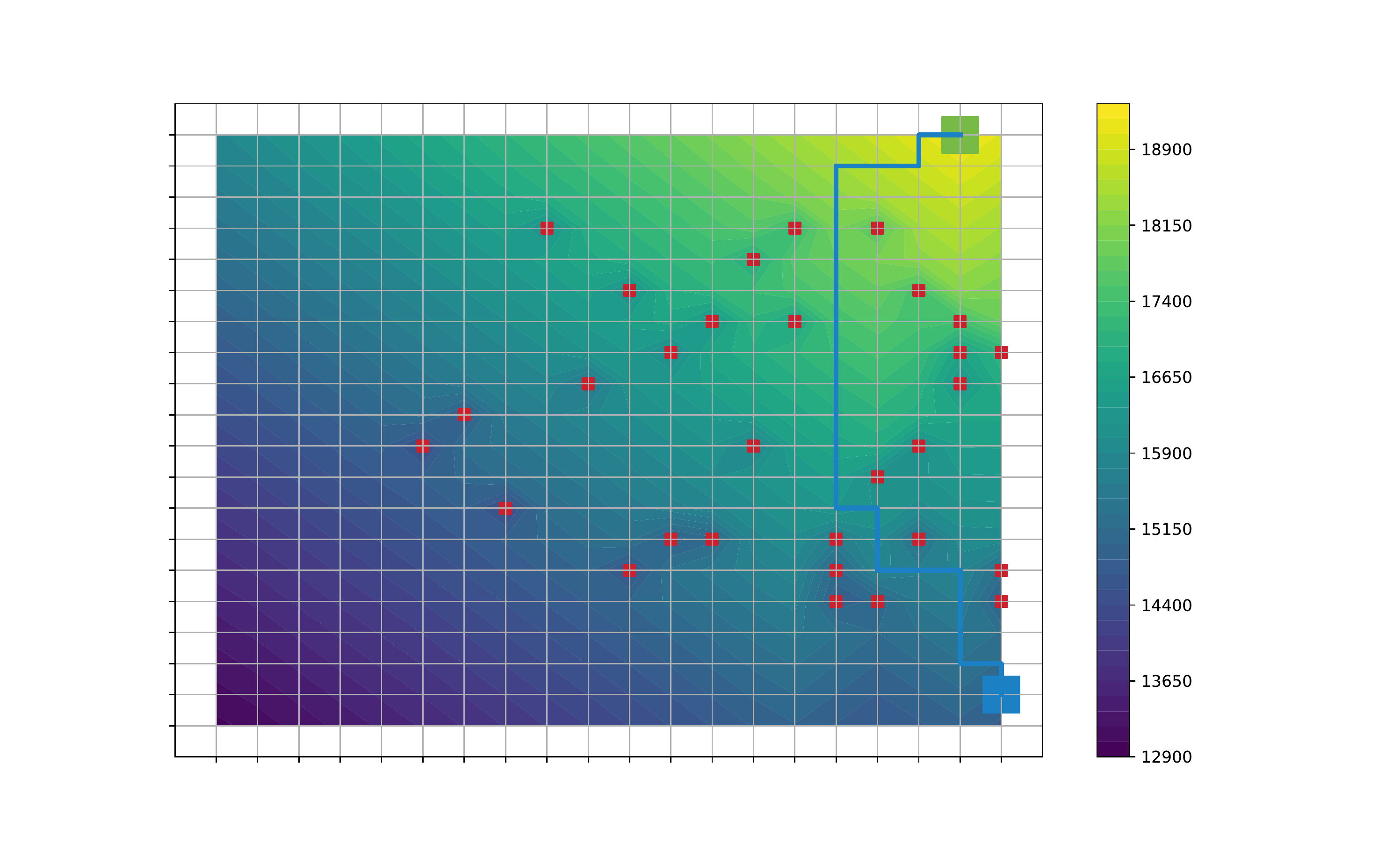}
     \subcaption{$P_{\pi^*}^s=79.65\%,E=40$}
   \end{minipage}
   \begin{minipage}[b]{.245\linewidth}
     \centering
     \includegraphics[width=1\textwidth]{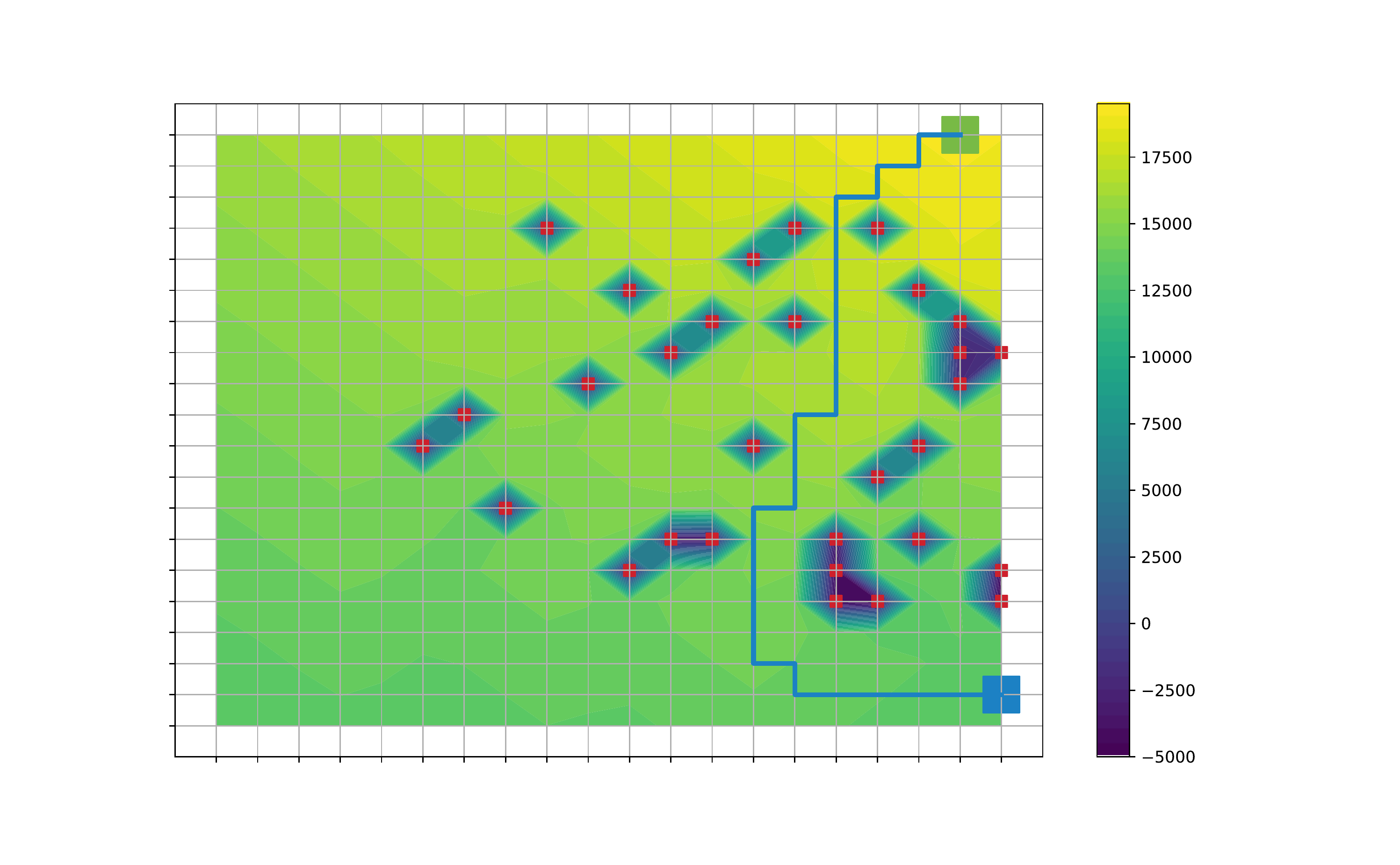}
     \subcaption{$P_{\pi^*}^s=93.55\%,E=20$}
   \end{minipage} 
   \begin{minipage}[b]{.245\linewidth}
     \centering
     \includegraphics[width=1\textwidth]{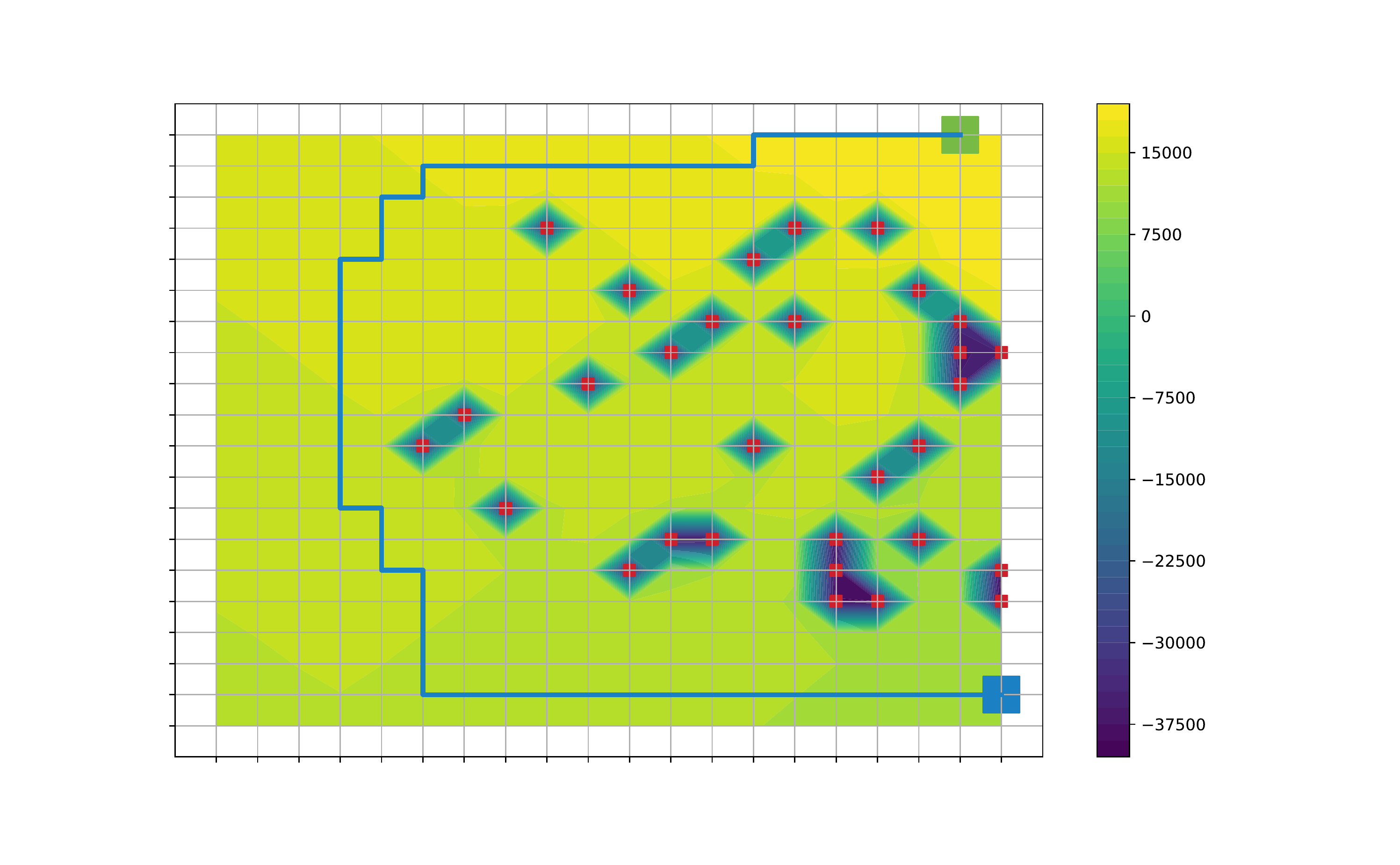}
     \subcaption{$P_{\pi^*}^s=99.90\%, E=5$}
   \end{minipage}    
 \caption{Filled contour maps of the value function and the learned control policy for the grid world robot navigation problem. With a lower upper bound on the cost constraint $(E)$, the policy is more risk-averse and achieves a higher probability of success for reaching the goal $P_{\pi^*}^s$. The shown path is one realization based on the learned policy.}
 \label{gridvalue}
\end{figure*}

\subsection{Solar Powered UAV-Based Wireless Networks}
As a second application domain, we study a wireless network management problem, which is of interest to the wireless networking research community. We consider a solar-powered UAV-based wireless network consisting of one UAV which relays data from network servers to $N$ wireless devices. Wireless devices are independently and uniformly deployed within a circular geographical area $\mathbb{A}$ with radius $R_c$. The UAV is deployed at the center of $\mathbb{A}$, and is  equipped  with solar panels which harvest solar energy to replenish its on-board battery storage. Because solar light is attenuated through the cloud cover, solar energy harvesting is highest at higher altitudes. System time is slotted into fixed-length discrete time units indexed by $t$. During a time slot $t$, the solar-powered UAV provides downlink wireless coverage to wireless devices on the ground. The UAV is equipped with an intelligent controller which at each time slot $t$, controls its altitude $z^t$, its directional antenna half power beamwidth angle $\theta_B^t$, and its transmission power $P_\text{TX}^t$ in dBm, based on its battery energy state and altitude, to maximize the wireless coverage probability for the worst case edge user, while ensuring energy sustainability of the solar-powered UAV. Note that the worst case edge user is the user which is farthest from the UAV. Due to distance dependent free-space path loss, the received signal at the edge user from the UAV is heavily attenuated, resulting in a low Signal to Noise (SNR) ratio and decreased coverage probability. Thus, by maximizing the coverage probability of the edge user, the communication performance is improved for every user in the network, albeit not proportionally. This wireless network control problem can be modeled as follows,
\subsubsection{Communication Model}
during time slot $t$, the antenna gain can be approximated by \cite{mozaffari2016efficient},
\begin{equation}
G_\text{Antenna}^t = 
\begin{cases} \frac{29000}{(\theta_B^t)2},~~~~  \frac{\theta_B^t}{2} \leq  \phi \leq \frac{\theta_B^t}{2} 
\\
g(\phi),~~~~\text{otherwise},
\end{cases}
\end{equation}
where $\theta_B^t$ is in degrees, $\phi$ is the sector angle, and $g(\phi)$ is the antenna gain outside the main lobe. The antenna gain in dB scale is $G_{dB}^t = 10\text{log}_{10}(G^t_\text{Antenna})$. The air-to-ground wireless channel can be modeled by considering Line-of-Sight (LoS) and non-Line-of-Site (NLos) links between the UAV and the users separately. The probability of occurrence of LoS for an edge user during time slot $t$ can be modeled by $P_{LoS}^t = \zeta \Big( \frac{180}{\pi} \psi^t -15 \Big)^\eta$, where $\zeta$ and $\eta$ are constant values reflecting the environment impact, and $\psi^t$ is the elevation angle between the UAV and the edge user during time slot $t$, $\psi^t = \text{tan}^{-1}(z^t/R_c)$ \cite{mozaffari2016efficient}. The probability of an NLoS link is $P_{NLoS}^t =1-P_{LoS}^t$. The shadow fading for the LoS and NLoS links during time slot $t$ can be modeled by normal random variables $\eta_{LoS}^t \sim \mathcal{N}(\mu_{LoS}, (\sigma^t_{LoS})^2)$ and $\eta_{NLoS}^t \sim \mathcal{N}(\mu_{NLoS}, (\sigma^t_{NLoS})^2)$, respectively, where the mean and variance are in dB scale and depend on on the elevation angle and environment parameters,  $\sigma_{LoS}^t(\psi^t) = k_1 \text{exp}(-k_2 \psi^t)$, $\sigma_{NLoS}^t(\psi^t) = g_1 \text{exp}(-g_2 \psi^t)$ \cite{mozaffari2016efficient}. Accordingly, the coverage probability for the edge user during time slot $t$, defined as the probability the received SNR by an edge user is larger than a threshold ($\text{SNR}_{\text{Th}}$) is,
\begin{equation}
\begin{aligned}
P_{\text{cov}}^t = &P_{LoS}^t \mathcal{Q}\Big( \frac{-P_{TX}^t -G^t_\text{dB} + \mu_{LoS} + L_{dB} + P_{min}}{\sigma_{LoS}} \Big) \\
&+ P_{NLoS}^t \mathcal{Q}\Big(\frac{-P_{TX}^t -G^t_\text{dB} + \mu_{NLoS} + L_{dB} + P_{min}}{\sigma_{NLoS}}  \Big) \\
\end{aligned}
\end{equation}
where $\mathcal{Q}\big(x \big)=1-P(X\leq x)$ for a standard normal random variable $X$ with mean $0$ and variance $1$, $P_{min} = 10\text{log}_{10}(\mathcal{N}_0\text{SNR}_\text{Th})$, $\mathcal{N}_0$ is the noise floor power, $L_{dB}$ is the distance-dependent path-loss, $L_{dB} = 10\alpha \text{log}\Big(\frac{4\pi f_0 \sqrt{(R_c)^2+(z^t)^2}}{c} \Big)$, $c$ is the speed of light, $f_0$ is the carrier frequency, and $\alpha$ is the path loss exponent. 
\subsubsection{UAV Energy Model}
the attenuation of solar light passing through a cloud is modeled by $\phi(d^{cloud}) = e^{-\hat{\beta}_c d^{cloud}}$, where $\hat{\beta}_c \geq 0$ denotes the absorption coefficient modeling the optical characteristics of the cloud, and $d^{cloud}$ is the distance that the solar light travels through the cloud \cite{sun2019optimal}. The solar energy harvesting model in time slot $t$ is, 
\begin{equation} \label{harvest}
    E_{\text{solar}}^t = \begin{cases}
    \tau \hat{S} \hat{G} \Delta t,~~~~~~~~~~\frac{z^t+z^{t+1}}{2} \geq z_{high} \\
    \tau \hat{S} \hat{G} \phi(z_{up}-z^n)\Delta t, z_{low} \leq \frac{z^t+z^{t+1}}{2} < z_{high} \\
    \tau \hat{S} \hat{G} \phi(z_{up}-z_{low})\Delta t,~\frac{z^t+z^{t+1}}{2} < z_{low}
    \end{cases}
\end{equation}
where $\tau$ is a constant representing the energy harvesting efficiency, $\hat{S}$ is  the area of solar panels, $\hat{G}$ is the average solar radiation intensity on earth, and $\Delta t$ is the time slot duration. $z_{high}$ and $z_{low}$ are the altitudes of upper and lower boundaries of the cloud. Based on this model, the output solar energy is highest above the cloud cover at $z_{high}$, and it attenuates exponentially through the cloud cover until $z_{low}$. During a time-slot $t$, the UAV can cruise upwards or downwards at a constant speed of $v_z^t$, or hover at the same altitude. The energy consumed for cruising or hovering during time slot $t$ is, 
\begin{equation}
\begin{aligned} \label{consumed}
E_{\text{UAV}}^t =& \left( \frac{W^2/(\sqrt{2}\rho A)}{\sqrt{2}V_h} \right) \Delta t \\
& + \left( Wv_z^t + P_{\text{static}} + P_{TX}^t \right) \Delta t, \\
&v_z^t = \frac{z^{t+1} - z^t}{\Delta t},~~~V_h = \sqrt{\frac{W}{2 \rho A}}
\end{aligned}
\end{equation}
Here, $W$ is the weight of the UAV, $\rho$ is air density, and $A$ is the total area of UAV rotor disks. $P_{\text{static}}$ is static power consumed for maintaining the operation of UAV. It is worth to mention that cruising upwards consumes more power than cruising downwards or hovering. Denote the battery energy storage of the UAV at the beginning of slot $t$ by $B^t$. The battery energy of the next slot is given by, 
\begin{equation}
    B^{t+1} = \text{max} \big\{0, \text{min} \{B^{t}+  E_{\text{solar}}^t - E_{\text{UAV}}^t, B_{\text{max}} \} \big\} 
\end{equation}

\subsubsection{CMDP Formulation}
this constrained control problem of maximizing the wireless coverage probability for the worst case edge user, while ensuring energy sustainability of the solar-powered UAV can be modeled as a discrete-time CMDP with discrete state-action spaces as follows. First, the altitude of the UAV is discretized into $N_z$ discrete units of $\frac{z_{max}-z_{min}}{N_z}$. Let the set of possible UAV altitudes be $\mathcal{Z} = \{z_{min}, z_{min}+\frac{z_{max}-z_{min}}{N_z}, z_{min} + 2 \frac{z_{max}-z_{min}}{N_z}, \cdots\}$. In addition, the finite battery energy of the UAV is discretized into $N_b$ energy units, where each energy unit is $e_u = \frac{B_{max}\times 60 \times 60}{N_b}$ Joules. Let $\mathcal{B} \in \{0, e_u,\cdots, (N_b-1)e_u\}$ be the set of possible UAV battery energy levels. Accordingly, the CMDP can be formulated as follows,

\begin{enumerate}
\item The state of the agent is the battery energy level and UAV altitude, $\forall s_t \in \mathcal{S}$, $s_t = (B^t, z^t)$, where  $B^t \in \mathcal{B}$ and $z^t  \in \mathcal{Z}$. Thus $\mathcal{S}=\mathcal{B}\times \mathcal{Z}$.
\item The agent controls the UAV vertical velocity $v_z^t \in \mathcal{A}_z$, the antenna transmission power $P_\text{TX}^t \in \mathcal{A}_{P_{TX}}$, and the half power beamwidth angle $\theta_B^t \in \mathcal{A} _{\theta_B}$. Thus, $\forall a_t \in \mathcal{A}$,  $\mathcal{A} =  \mathcal{A}_z \times \mathcal{A}_{P_{TX}} \times \mathcal{A}_{\theta_B}$, where $\mathcal{A}_z =\{-4,0,4\}$ m/s, $\mathcal{A}_{P_{TX}} = \{34, 38\}$ dBm, and $\mathcal{A} _{\theta_B} = \{28^o,56^o\}$.
\item The immediate reward is the wireless coverage probability for the edge user during time slot $t$, $\mathcal{R}(s_t,a_t) = P_\text{cov}^t$.

\item The immediate cost is the change in the battery level, $\mathcal{C}(s_t,a_t) = B^t-B^{t+1}$.

\item  $E = -\Delta B$, where  $\Delta B$  is the minimum desired battery energy increase over the initial battery energy.

\item Model dynamics \big\{$P(s_{t+1}|s_t,a_t),\forall s,a$\big\}: battery evolution of the UAV is modeled as an $M/D/1/N_b$ queue.  Energy arrival is according to a poisson process with rate $\lambda^t = \frac{E_\text{solar}^t}{e_u}$. Energy arrivals which see a full battery are rejected, and do not further influence the system. The finite battery size acts as a regulator on the queue size. Energy departure is deterministic and depends on the action taken by the controller. Hence, energy departure rate  is $\mu_t = \frac{E^t_{UAV}(a_t)}{e_u}$. Utilization factor of the battery is therefore $\rho_B^t = \frac{\lambda_t}{\mu_t}=\frac{E_\text{solar}^t}{E^t_\text{UAV}(a_t)}$. Based on poisson energy arrivals, the probability of $k$ arrivals during a one unit energy departure given action $a_t$ is taken,
\begin{equation} \label{ch4:model}
P\{k \text{~energy arrivals} |a_t \} =\frac{(\rho_B^t)^ke^{-\rho_B^t}}{k!}
\end{equation}
On the other hand, altitude state evolution is determinsitic based on $v_z^t$. Hence, the transition probability function of the CMDP can be derived based on \eqref{ch4:model} and the probability transition matrix of the embedded Markov chain with action $a_t$ for an $M/D/1/N_b$ queue  \cite{garcia2002transient}.
\end{enumerate}

This wireless communication system exhibits a trade-off between the altitude of the UAV and the coverage probability for the edge user. When the UAV hovers at a higher altitude, it can harvest more solar energy to replenish its on-board battery. However, at higher altitudes, the wireless coverage probability is worse due to signal attenuation. An optimal control policy should be learned to maximize the coverage probability for the edge user while ensuring the UAV's battery is not depleted. 
\begin{figure}
\centering
   \begin{minipage}[b]{.66\linewidth}
     \centering
     \includegraphics[width=1\textwidth]{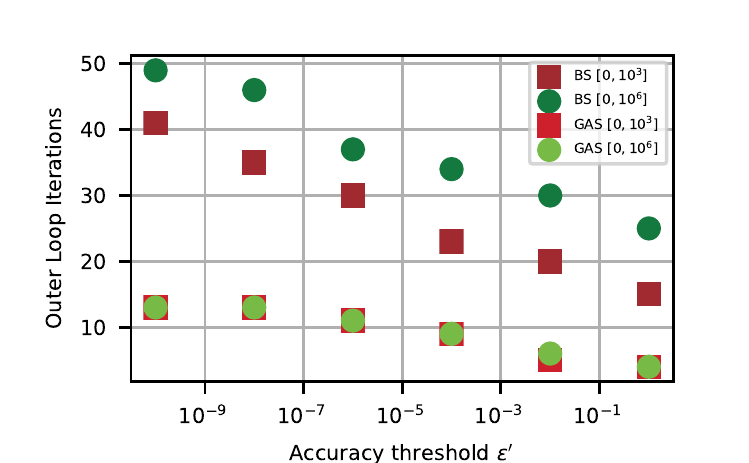}
     \subcaption{Comparison with binary search}
   \end{minipage} 
   \begin{minipage}[b]{.66\linewidth}
     \centering
     \includegraphics[width=1\textwidth]{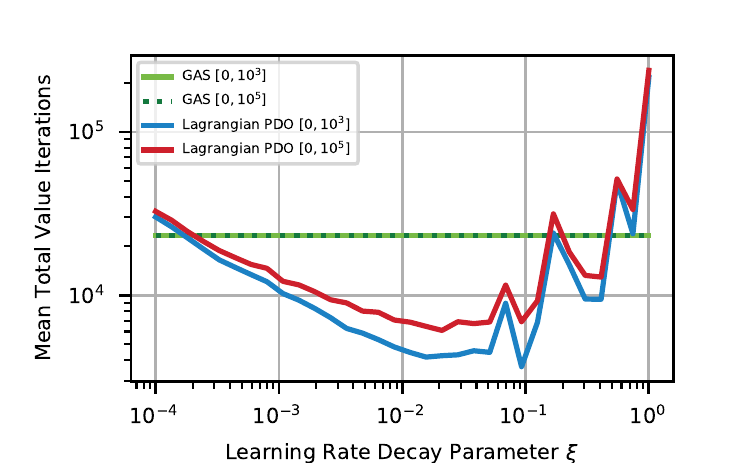}
     \subcaption{Comparison with Lagrangian approach}
   \end{minipage} 
 \caption{Performance comparison on the solar powered UAV-Based wireless network management problem}
 \label{uavcomp}
\end{figure}
\begin{figure*}[ht]
   \begin{minipage}[b]{.245\linewidth}
     \centering
     \includegraphics[width=1\textwidth]{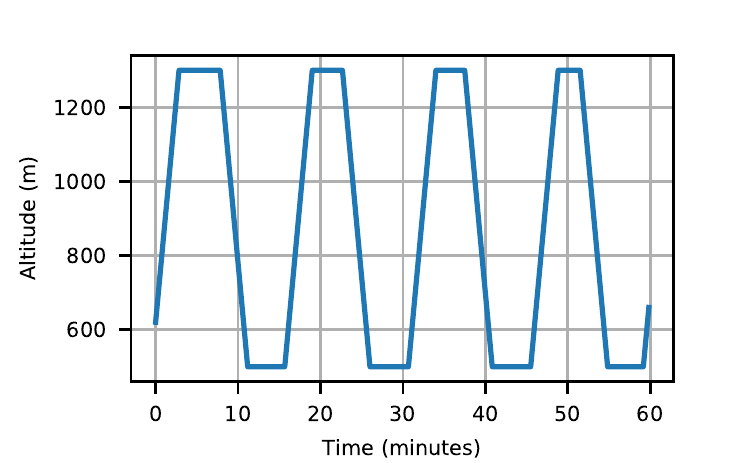}
     \subcaption{UAV altitude}
   \end{minipage} 
   \begin{minipage}[b]{.245\linewidth}
     \centering
     \includegraphics[width=1\textwidth]{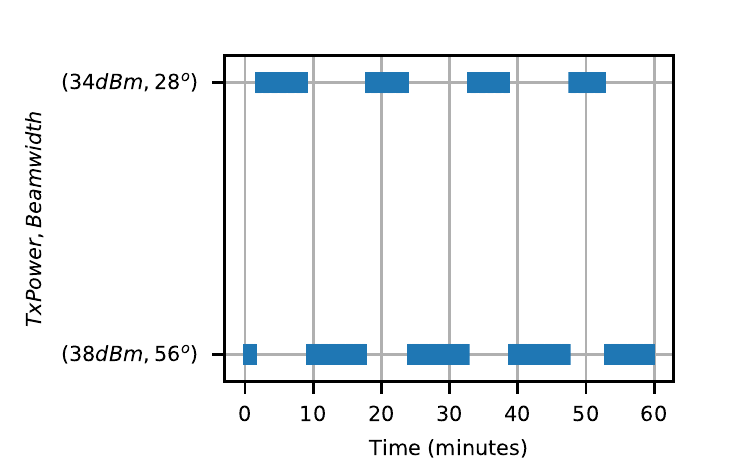}
     \subcaption{Tx power and beamwidth}
   \end{minipage} 
    \begin{minipage}[b]{.245\linewidth}
     \centering
     \includegraphics[width=0.95\textwidth]{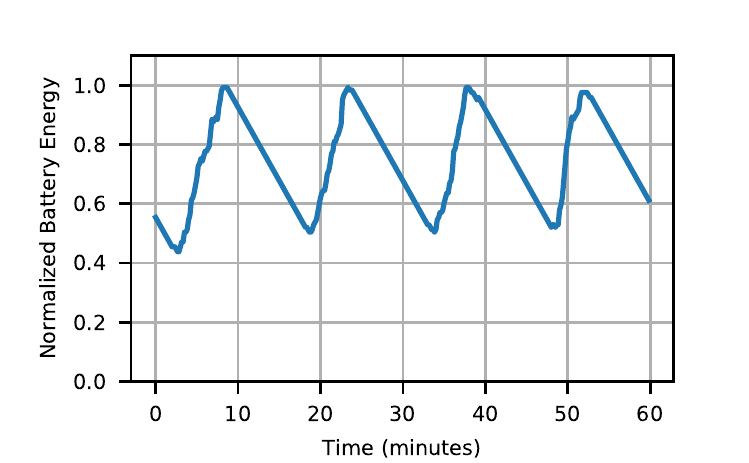}
     \subcaption{Battery energy evolution}
   \end{minipage}
   \begin{minipage}[b]{.245\linewidth}
     \centering
     \includegraphics[width=1\textwidth]{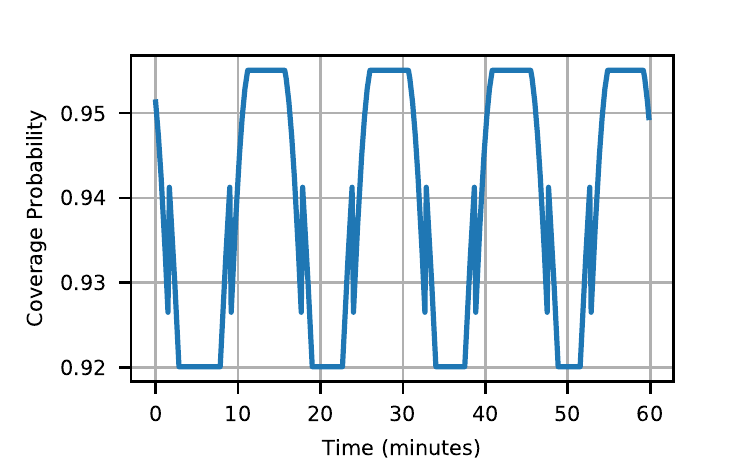}
     \subcaption{Edge user coverage probability}
   \end{minipage}    
 \caption{Learned control policy for the solar powered UAV-Based wireless network management problem}
 \label{uavresults}
\end{figure*}
\subsubsection{Experiment Results} simulation parameters for this experiment are outlined in Table \ref{table:simparam}. Based on the chosen discretization levels $N_z$ and $N_b$, the CMDP has $|\mathcal{S}|=3025$ states and $|\mathcal{A}|=12$. 
\begin{table}[ht]  
\caption{Simulation parameters for the solar powered UAV-Based wireless network management problem}   
\centering                          
\begin{tabular}{c c | c c }            
\hline\hline                        
Parameter  & Value &  Parameter & Value \\ [0.5ex] 
\hline                              
$N_z$  & $121$  & $N_b$ & $25$    \\
$\alpha$  & $2.5$  & $\tau$ & $0.4$    \\
$B_{max}$ & $100$ Wh & $\Delta B$ & $1.67$ Wh \\ 
$\Delta t$ & $10s$    & $\Tilde{S}$ & $1m^2$   \\
$f_0$ & $2$ GHz  &$\Tilde{G}$ & $1367 W/m^2$   \\
$g(\phi)$ & $0$   &$W$ &  $39.2 kg*m/s^2$   \\
$n_0$ & $-100dBm$ & $\rho$ &  $1.225kg/m^3$ \\
$N_b$ & $25$ &  $A$ &  $0.18m^2$\\
$N_z$ & $121$  & $P_{static}$ &  $5$ watts \\
$\text{SNIR}_\text{Th}$ & $5$  & $R_c$  & $250$ m \\
$\hat{\beta}_c$ & $0.01$ & $k_1,~k_2$ & $10.39,~0.05$ \\
$z_{high},z_{low}$ & $1.3,0.7$km & $\Delta z_{min},\Delta z_{max}$ & $-40m,40m$  \\
$z_{min},z_{max}$ & $0.5,1.5$km &  $g1,~g2$ & $29.06,~0.03$  \\
$\mu_{LOS},~\mu_{NLOS}$ & $1,~20dB$ & $\zeta$ &  0.6 \\
$\gamma$ & $0.99$  &  $\eta$ & $0.11$   \\
\hline                              
\end{tabular}          \label{table:simparam}   
\end{table} 
In Figures \ref{uavcomp}(b) and \ref{uavcomp}(c), we compare the convergence performance of the proposed GAS with BS and the Lagrangian PDO approach, respectively. As previously noted from Figures \ref{gridresults}(b) and \ref{gridresults}(c), it can be see that the proposed algorithm compares favourably to BS and Lagrangian PDO, despite the increased problem size.  From Table \ref{LP}, we 
can observe that both the proposed GAS and Gurobi's LP solver converge to the same $\mu^*$, although GAS achieves a lower Bellman error in the estimation of the value function.

The learned policy by our proposed GAS algorithm is shown in Figure \ref{uavresults}. It can be seen from Figure \ref{uavresults}(a) that the agent learns an adaptive policy in which the UAV climbs up to recharge its on-board battery, and then climbs down when the battery is full to improve the coverage probability for the worst case edge user, as can be seen from Figures \ref{uavresults}(c) and  \ref{uavresults}(d). In addition, Figure \ref{uavresults}(b) shows the learned control policy for the transmission power and half power beamwidth angle. When the UAV is up to charge its battery, the lower transmission power and smaller beamwidth angle are selected. On the other hand, when the UAV is down to improve the coverage probability, the higher transmission power and larger beamwidth angle are selected. This is because with a smaller beamwidth angle, the antenna gain is higher, and a lower transmission power is required to counter the effects of large scale fading.

\section{Conclusion}
In this brief, we have proved that the optimization objective in the dual linear programming formulation of a constrained Markov Decision process (CMDP) is a piece-wise linear convex function (PWLC) with respect to the Lagrange penalty multiplier. Based on this result, a novel two-level Gradient-Aware Search (GAS) algorithm which exploits the PWLC structure has been proposed to find the optimal state-value function and Lagrange penalty multiplier of a CMDP. We have applied the proposed algorithm on two different problems, and compared its performance with binary search, Lagrangian primal-dual optimization, and linear programming. Compared with existing algorithms, it has been shown that our proposed algorithm converges to the optimal solution faster, does not require hyper-parameter tuning, and is not sensitive to initialization of the Lagrange penalty multiplier. In our future work, we will study the extension of the proposed GAS algorithm to the model-free reinforcement learning problem.

\section*{Acknowledgment}
This work was supported in part by the NSF grants ECCS-1554576 and ECCS-1610874. We gratefully acknowledge the computing resources provided on Bebop, a high-performance computing cluster operated by the Laboratory Computing Resource Center at Argonne National Laboratory.

\bibliographystyle{IEEEtran}
\bibliography{Bibliography}
\end{document}